\documentclass[journal]{IEEEtran}

\usepackage{amsmath,amssymb,amsbsy,amsfonts,amscd,bm,url,color,latexsym}
\usepackage[hidelinks]{hyperref}
\usepackage{algorithm}
\usepackage{algorithmic}
\usepackage{comment}
\usepackage{multirow}
\usepackage{enumitem}
\usepackage{fancyhdr}
\usepackage{tcolorbox}
\usepackage{wrapfig}
\usepackage{cleveref}
\usepackage{subcaption}
\usepackage{authblk}

\newtheorem{lemma}{Lemma}
\newtheorem{definition}{Definition}

\newtheorem{theorem}{Theorem}

\newcommand{\R}{\mathbb{R}}

\newcommand{\e}{\begin{equation}}
\newcommand{\ee}{\end{equation}}
\newcommand{\en}{\begin{equation*}}
\newcommand{\een}{\end{equation*}}
\newcommand{\eqn}{\begin{eqnarray}}
\newcommand{\eeqn}{\end{eqnarray}}
\newcommand{\bmat}{\begin{bmatrix}}
\newcommand{\emat}{\end{bmatrix}}

\DeclareMathAlphabet\mathbfcal{OMS}{cmsy}{b}{n}

\newcommand{\vct}[1]{\boldsymbol{#1}}
\newcommand{\mtx}[1]{\boldsymbol{#1}}

\newcommand{\bmtx}[1]{\mathbf{#1}}

\newcommand{\<}{\langle}
\renewcommand{\>}{\rangle}

\def \st {\operatorname*{s.t.\ }}

\newcommand{\wh}{\widehat}

\newcommand{\calC}{\mathcal{C}}

\newcommand{\calH}{\mathcal{H}}

\newcommand{\calP}{\mathcal{P}}

\newcommand{\va}{\vct{a}}

\newcommand{\vx}{\vct{x}}
\newcommand{\vy}{\vct{y}}

\newcommand{\mA}{\mtx{A}}
\newcommand{\mB}{\mtx{B}}
\newcommand{\mC}{\mtx{C}}

\newcommand{\mH}{\mtx{H}}

\newcommand{\mR}{\mtx{R}}

\newcommand{\mW}{\mtx{W}}
\newcommand{\mX}{\mtx{X}}
\newcommand{\mY}{\mtx{Y}}

\newcommand{\mTheta}{\mtx{\Theta}}

\newcommand{\mId}{\bmtx{I}}

\newlength{\imgwidth}
\hypersetup{
    colorlinks=true,%
    citecolor=blue,%
    filecolor=blue,%
    linkcolor=blue,%
    urlcolor=blue
}

\renewcommand{\mathbf}{\boldsymbol}

\def \endprf{\hfill {\vrule height6pt width6pt depth0pt}\medskip}
\newenvironment{proof}{\noindent {\bf Proof} }{\endprf\par}

\ifCLASSINFOpdf
\usepackage{url}

\usepackage{graphicx}

\usepackage[utf8]{inputenc} 
\usepackage[T1]{fontenc}    
\usepackage{hyperref}       
\usepackage{url}            
\usepackage{booktabs}       
\usepackage{amsfonts}       
\usepackage{nicefrac}       
\usepackage{microtype}      
\usepackage{xcolor}         
\usepackage{epsfig}
\usepackage{epstopdf}

\usepackage{balance}

\usepackage{bm}

\usepackage[numbers,sort&compress]{natbib}

\usepackage{array}
\usepackage{diagbox}
\usepackage{pdfsync}

\makeatletter
\def\@IEEEsectpunct{\ \,}
\def\paragraph{\@startsection{paragraph}{4}{\z@}{1.5ex plus 1.5ex minus 0.5ex}%
{0ex}{\bfseries}}
\makeatother

\begin{document}

\title{Convergence Analysis for Learning Orthonormal Deep Linear Neural Networks}

\author{Zhen Qin,  Xuwei Tan, and Zhihui Zhu, \IEEEmembership{Member, IEEE}
\thanks{Z. Qin, X. Tan and Z. Zhu are with the Department of Computer Science and Engineering, the Ohio State University, OH 43210, USA (e-mail: \{qin.660, tan.1206, zhu.3440\}@osu.edu).

This work was partly supported by NSF grants CCF-2240708 and CCF-2241298.
}
}

\maketitle

\begin{abstract}
Enforcing orthonormal or isometric property for the weight matrices has been shown to enhance the training of deep neural networks by mitigating gradient exploding/vanishing and increasing the robustness of the learned networks. However, despite its practical performance, the theoretical analysis of orthonormality in neural networks is still lacking; for example, how orthonormality affects the convergence of the training process. In this letter, we aim to bridge this gap by providing convergence analysis for training orthonormal deep linear neural networks.
Specifically, we show that Riemannian gradient descent with an appropriate initialization converges at a linear rate for training orthonormal deep linear neural networks with a class of loss functions. Unlike existing works that enforce orthonormal weight matrices for all the layers, our approach excludes this requirement for one layer, which is crucial to establish the convergence guarantee. Our results shed light on how increasing the number of hidden layers can impact the convergence speed. Experimental results validate our theoretical analysis.
\end{abstract}


\IEEEpeerreviewmaketitle

\section{Introduction}
\label{Intro}

\IEEEPARstart{E}{forcing} orthonormal or isometric properties of the weight matrices has numerous advantages for the practice of deep learning: (i) it provides a better initialization \cite{mishkin2015all,saxe2013exact}, (ii) it mitigates the problem of exploding/vanishing gradients during training \cite{bengio1994learning, pascanu2013difficulty,le2015simple,arjovsky2016unitary,hanin2018neural}, (iii) the resulting {\it orthonormal neural networks} \cite{saxe2013exact,harandi2016generalized,li2019orthogonal,huang2020controllable,liefficient20LCLR,wang2020orthogonal,malgouyres2022existence,dorobantu2016dizzyrnn,mhammedi2017efficient,vorontsov2017orthogonality} exhibit improved robustness \cite{cisse2017parseval} and reduced overfitting issues \cite{cogswell2015reducing}.

Various approaches have been proposed for training neural networks, mainly falling into two categories: soft orthonormality and hard orthonormality. The first category of methods, such as those in \cite{cisse2017parseval, huang2020controllable, wang2020orthogonal, bansal2018can}, adds an additional orthonormality regularization term to the training loss, resulting in weight matrices that are approximately orthonormal. In contrast, the other methods, as found in \cite{huang2018orthogonal, li2019orthogonal, liefficient20LCLR}, learn weight matrices that are exactly orthonormal through the use of Riemannian optimization algorithms on the Stiefel manifold.

While orthonormal neural networks demonstrate strong practical performance, there remains a gap in the theoretical analysis of orthonormality in neural networks. For example, convergence analysis for training neural networks has been extensively studied \cite{bartlett2018gradient, arora2018convergence,zou2020global,shamir2019exponential,allen2019convergence, chatterjee2022convergence,zhou2021local,zhang2019learning}. However, all these results focus on standard training without orthonormal constraints, making them inapplicable to the training of orthonormal neural networks. To the best of our knowledge, there is a lack of rigorous convergence analysis even for orthonormal deep linear neural networks (ODLNNs).
Despite its linear structure, a deep linear neural network still presents a non-convex training problem and has served as a testbed for understanding deep neural networks \cite{bartlett2018gradient,zou2020global,arora2018convergence,shamir2019exponential}.
In this letter, we aim to understand the effect of the orthonormal structure on the training process by studying ODLNN.

\vspace{-0.2cm}
\paragraph*{Our contribution}  Specifically, we provide a local convergence rate of Riemannian gradient descent (RGD) for training the ODLNN. To achieve this, unlike existing works \cite{cisse2017parseval,li2019orthogonal,liefficient20LCLR} that impose orthonormal constraints on all the weight matrices, we exclude such a constraint for one layer (say the weight matrix in the first hidden layer). The exclusion of a specific layer plays a crucial role in analyzing the convergence rate.
Our findings demonstrate that within a specific class of loss functions, adhering to the restricted correlated gradient condition \cite{Han20}, the RGD algorithm exhibits linear convergence speed when appropriately initialized. Notably, our results also indicate that as the number of layers in the network increases, the rate of convergence only experiences a polynomial decrease. The validity of our theoretical analysis has been confirmed by experiments.

{\bf Notation}: We use bold capital letters (e.g., $\bm{A}$) to denote matrices,  bold lowercase letters (e.g., $\bm{a}$) to denote vectors, and italic letters (e.g., $a$) to denote scalar quantities.
The superscript $(\cdot)^\top$ denotes the transpose.
$\|\mA\|$ and $\|\mA\|_F$ respectively represent the spectral norm and Frobenius norm of $\mA$. $\sigma_{\min}(\mA)$ is the smallest singular value of $\mA$. The condition number of $\mA$ is defined as $\kappa(\mA) = \frac{\|\mA\|}{\sigma_{\min}(\mA)}$. $\|\va\|_2$ is the $l_2$ norm of $\va$. For a positive integer $K$, $[K]$ denotes the set $\{1,\dots, K \}$. $b = \Omega(a)$ represents $b\ge ca$ for some universal constant $c$.

\vspace{-0.3cm}
\section{Riemannian Gradient Descent  for Orthonormal Deep Linear Neural Networks}
\label{Deep Linear Neural Model}

\paragraph*{Problem statement} Given a training set $\{ (\vx_i, \vy_i^\star) \}_{i=1}^n\in \R^{d_x}\times \R^{d_y} $, our goal is to estimate a hypothesis (predictor) from a parametric family $\calH := \{ h_{\theta}: \R^{d_x}\ \to \R^{d_y} | \theta\in\mTheta \} $ by minimizing the following empirical risk:
\begin{eqnarray}
\label{Deep  Neural Model - general model}
    \min_{\theta\in\mTheta} g(\theta) = \frac{1}{n}\sum_{i=1}^{n} l(h_{\theta}(\vx_i);\vy_i^\star),
\end{eqnarray}
where $l(h_{\theta}(\vx_i);\vy_i^\star)$ is a suitable loss that captures the difference between the network prediction $h_{\theta}(\vx_i)$ and the label $\vy_i^\star$.
For convenience, we stack all the training samples together as $\mX = \begin{bmatrix}\vx_1 & \cdots & \vx_n \end{bmatrix}$ and $\mY^\star = \begin{bmatrix}\vy_1^\star & \cdots & \vy_n^\star \end{bmatrix}$.

Our main focus is on orthonormal deep linear neural networks (ODLNNs), which are fully-connected neural networks of form $h_\theta(\vx_i) = \mW_N \cdots \mW_1 \vx_i$ with $\mW_i\in\R^{d_i\times d_{i-1}}$ for $i\in[N]$, where $d_0 = d_x$ and $d_N = d_y$. In ODLNNs, we further assume that the weight matrices to be row orthogonal or column orthogonal depending on the dimension. Without loss of generality, we assume that all the matrices $\{\mW_i\}_{i\geq 2}$ are column orthonormal, except for $\mW_1$. This is different to the previous works \cite{cisse2017parseval,li2019orthogonal,liefficient20LCLR} which impose orthonormal constraints on all the weight matrices.  Allowing $\mathbf{W}_1$ to be unstructured offers more flexibility, as otherwise $\mathbf{W}_N\cdots \mathbf{W}_1$ can only represent an orthonormal matrix, which restricts the output the same energy as input (i.e., $\|\vy_i^\star\|_2= \|\vx_i\|_2$). The choice of free weight matrix can vary, and the following analysis would still hold. Now the training loss can be written as
\begin{eqnarray}
    \label{DLN_LOSS_FUNCTION_1 general}
    \begin{split}
 \min_{\mbox{\tiny$\begin{array}{c}
     \mW_i\in\R^{d_{i}\times d_{i-1}}\\
     i\in [N]\end{array}$}} & \!\!\!\! g(\mW_N,\dots, \mW_1) =  L(\mW_N \cdots \mW_1\mX;\mY^\star),\\
    &\!\! \st \ \mW_i^\top\mW_i=\mId_{d_{i-1}}, \ \ i=2,\dots N,
    \end{split}
\end{eqnarray}
where $L$ denotes a loss function encompassing all samples.

\begin{definition} [Data model]
Following the previous work on deep linear neural networks \cite{bartlett2018gradient,arora2018convergence}, we assume that the dataset $\mX$ is whitened, i.e., its empirical covariance matrix is an identity matrix as $\mX\mX^\top = \mId_{d_x}$. Also assume that the output is generated by a teacher ODLNN model\footnote{Here, for the sake of simplifying the subsequent analysis, we designate the teacher model as the ODLNN, which can encompass any linear model.}, i.e.,  $\mY^\star = \mW_N^\star \cdots \mW_1^\star \mX$, where $\mW_i^\star\in \R^{d_i\times d_{i-1}}$ and
$\{\mW_i^\star\}_{i\ge 2}$ are column orthonormal matrices.
\label{def:data-model}\end{definition}

\vspace{-0.25cm}
\paragraph*{Stiefel manifold}

The Stiefel manifold $\text{St}(m,n)=\{\mC\in\R^{m\times n}: \mC^\top\mC=\mId_{n}\}$  is a Riemannian manifold that is composed of all $m\times n$ orthonormal matrices. We can regard $\text{St}(m,n)$ as an embedded submanifold of a Euclidean space and further define $\text{T}_{\mC} \text{St}:=\{\mA\in\R^{m\times n}: \mA^\top \mC+\mC^\top\mA={\bm 0} \}$ as its tangent space at the point $\mC\in\text{St}(m,n)$.
For any $\mB\in\R^{m\times n}$, the projection of $\mB$ onto $\text{T}_{\mC} \text{St}$ is given by \cite{LiSIAM21}
\begin{eqnarray}
\label{projection on the tangent space of Stiefel}
    \calP_{\text{T}_{\mC} \text{St}}(\mB)=\mB-\frac{1}{2}{\mC}(\mB^\top\mC+\mC^\top\mB),
\end{eqnarray}
and its orthogonal complement is $\calP_{\text{T}_{\mC} \text{St}}^{\perp}(\mB)= \mB - \calP_{\text{T}_{\mC} \text{St}}(\mB) $$ = \frac{1}{2}\mC(\mB^\top\mC+\mC^\top\mB)$. When we have a gradient $\mB$ defined in the Hilbert space, we can use the projection operator \eqref{projection on the tangent space of Stiefel} to compute the Riemannian gradient $\calP_{\text{T}_{\mC} \text{St}}(\mB)$ on the tangent space of the Stiefel manifold. To project $\wh \mC = \mC - c \calP_{\text{T}_{\mC} \text{St}}(\mB)$ with any positive constant $c$  back onto the Stiefel manifold, we can utilize the polar decomposition-based retraction, i.e.,
\vspace{-0.08cm}
\begin{eqnarray}
    \label{polar decomposition-based retraction}
    \text{Retr}_{\mC}(\wh \mC)=\wh \mC(\wh \mC^\top \wh \mC)^{-\frac{1}{2}}.
\end{eqnarray}

\paragraph*{Riemannian gradient descent (RGD)} Given the gradient $\nabla_{\mW_i}g({\mW_N^{(t)}},\dots, {\mW_1^{(t)}})$,  we can compute the Riemannian gradient $\calP_{\text{T}_{\mW_i} \text{St}}\big(\nabla_{\mW_i}g({\mW_N^{(t)}},\dots, {\mW_1^{(t)}}) \big)$  on the Stiefel manifold via \eqref{projection on the tangent space of Stiefel}.
To streamline the notation, let us represent $\nabla_{\mW_i}g({\mW_N^{(t)}},\dots, {\mW_1^{(t)}})$ as $\nabla_{\mW_i^{(t)}}g$. Now the weight matrices can be updated via the following RGD:
\begin{equation}
\begin{split}
    & \mW_1^{(t+1)} ={{\mW_1}^{(t)}}-\mu\gamma\nabla_{\mW_1^{(t)}}g,\\
    &\mW_i^{(t+1)} =\text{Retr}_{\mW_i}\Big({{\mW}_i^{(t)}}-\mu\calP_{\text{T}_{\mW_i} \text{St}}\big(\nabla_{\mW_i^{(t)}}g \big) \Big), \ i \ge 2,
\end{split}
\label{SGD_GRADIENT_DESCENT_1_2}
\end{equation}
where $\mu>0$ is the learning rate for $\{\mW_i\}$ and $\gamma>0$ controls the ratio between the learning rates for $\mW_1$ and $\{\mW_i\}_{i\ge 2}$.
The discrepant learning rates in \eqref{SGD_GRADIENT_DESCENT_1_2} are used to accelerate the convergence rate of $\mW_1$ since the energy of $\mW_1^\star$ and $\{\mW_i^\star\}_{i\ge 2}$ are unbalanced, i.e., $\|\mW_1^\star\|^2 = \|\mY^\star\|^2$ and $\|\mW_i^\star\|^2=1$ when the dataset $\mX$ is whitened, i.e. $\mX\mX^\top = \mId_{d_x}$.

\vspace{-0.3cm}
\section{Convergence Analysis}
\label{Sec: Convergence analysis}

\vspace{-0.1cm}
In this section, we will delve into the convergence rate analysis of RGD for training ODLNNs. Towards that goal, we will use the teacher model introduced in \Cref{def:data-model} that the training samples $(\mX,\mY^\star)$ are generated according to $\mY^\star = \mW_N^\star \cdots \mW_1^\star \mX$. Given the nonlinear nature of the retraction operation in the RGD, we will study the convergence in terms of the weight matrices ${\rm W} = \{ \mW_i \}$ and ${\rm W}^\star = \{{\mW_i^{\star}} \}$. But we will show that the convergence can be equivalently established in terms of the outputs.

\vspace{-0.15cm}
\paragraph*{Distance measure}  Consider that the factors in  ${\rm W}^\star$ are identifiable up to orthonormal transforms since $\mW_N^\star\cdots\mW_1^\star\mX = \mW_N^\star\mR_{N-1} \mR_{N-1}^\top\mW_{N-1}^\star \mR_{N-2} \cdots \mW_1^\star\mX$ for any orthonormal matrices $\mR_i\in\mathbb{O}^{d_{i}\times d_{i}},i\in[N-1]$. Also, $\mW_1^\star$ and $\mW_i^\star$ could be imbalanced as $\mW_i^\star$ are orthonormal for $i\ge 2$. This discrepancy can be quantified by observing that $\|\mW_1^\star\| = \|\mY^\star\|$ (since the input matrix $\mX$ is whitened) and $\|{\mW_i^{\star}}\| = 1$ for all $i \ge 2$.
Thus, we propose the following measure to capture the distance between two sets of factors:
\vspace{-0.1cm}
\begin{eqnarray}
\label{BALANCED NEW DISTANCE BETWEEN TWO Weight matrices}
\text{dist}^2({\rm W},{\rm W}^\star)\! = \!\!\!\!\!\!\min_{\mR_i\in\mathbb{O}^{d_{i}\times d_{i}}, \atop i\in[N-1]}&\!\!\!\!\!\!\!\!\!\!\!\!&\sum_{i=2}^{N} \|\mY^\star\|^2\|\mW_i-\mR_{i}^\top{\mW_i^{\star}}\mR_{i-1}\|_F^2\nonumber\\
    &\!\!\!\!\!\!\!\!\!\!\!\!\!& + \|\mW_1 - \mR_1^\top\mW_1^\star\|_F^2.
\end{eqnarray}
Here the coefficient $\|\mY^\star\|^2$ is to harmonize the energy levels between ${\rm W} = \{ \mW_i \}_{i\geq 2}$ and $\mW_1$. The following result elucidates the connection between $\text{dist}^2({\rm W},{\rm W}^\star)$ and $\|\mY-\mY^\star\|_F^2$, guaranteeing the convergence of $\mY$ as ${\rm W}$ approaches the global minima.
\begin{lemma}
\label{LOWER BOUND OF TWO DISTANCES IN DNN}
Assume a whitened input $\mX\in\R^{d_x\times n}$, i.e. $\mX\mX^\top = \mId_{d_x}$.
Let $\mY \!=\! \mW_N \cdots \mW_1 \mX$ and $\mY^\star \!=\! \mW_N^\star\cdots \mW_1^\star \mX$,  where  $\mW_i, \mW_i^\star\in\R^{d_i \times d_{i-1}}$ are orthonormal for $i=2,\dots,N$. Given that $\|\mW_1\|^2\leq\frac{9\|\mY^\star\|^2}{4}$ and $\|\mW_1^\star\|^2 = \|\mY^\star\|^2$, we can get
\vspace{-0.1cm}
\begin{eqnarray}
    \label{LOWER BOUND OF TWO DISTANCES IN DNN_1}
    \|\mY-\mY^\star\|_F^2&\!\!\!\!\geq\!\!\!\!&\frac{1}{(16N-8)\kappa^2(\mY^\star)}\text{dist}^2({\rm W},{\rm W}^\star),\\
    \label{UPPER BOUND OF TWO DISTANCES IN DNN_1}
    \|\mY-\mY^\star\|_F^2&\!\!\!\!\leq\!\!\!\!&\frac{9N}{4}\text{dist}^2({\rm W},{\rm W}^\star).
\end{eqnarray}
\end{lemma}
\begin{proof}
Using the result \cite[eq. (E.4)]{Han20}, for any $j\ge 2$, we have that  $\|\mW_N\cdots \mW_j - \mW_N^\star \cdots $ $ \mW_j^\star \mR_{j-1}\|_F^2\!\!\leq\!\! \frac{4\|\mY - \mY^\star\|_F^2}{\sigma_{\min}^2(\mY^\star)}$ for any $\mR_i\in\mathbb{O}^{d_{i}\times d_{i}}$. It follows that
\vspace{-0.1cm}
\begin{eqnarray}
    \label{LOWER BOUND OF TWO DISTANCES IN DNN_5}
    &\!\!\!\!\!\!\!\!&\|\mW_{N-1} - \mR_{N-1}^\top\mW_{N-1}^\star\mR_{N-2}\|_F^2 \nonumber\\
    &\!\!\!\!=\!\!\!\!& \|\mW_N^\star\mR_{N-1}\mW_{N-1} - \mW_{N}^\star\mW_{N-1}^\star\mR_{N-2}\|_F^2\nonumber\\
    &\!\!\!\!\leq\!\!\!\!&2\|\mW_N^\star\mR_{N-1} - \mW_N\|_F^2\|\mW_{N-1}\|^2\nonumber\\
    &\!\!\!\!\!\!\!\!& +2\|\mW_N\mW_{N-1} - \mW_{N}^\star\mW_{N-1}^\star\mR_{N-2}\|_F^2\nonumber\\
    &\!\!\!\!\leq\!\!\!\!&\frac{16\|\mY - \mY^\star\|_F^2}{\sigma_{\min}^2(\mY^\star)}.
\end{eqnarray}
Similarly, we get $\|\mW_{i}  -  \mR_{i}^\top\mW_{i}^\star\mR_{i-1}\|_F^2\leq \frac{16\|\mY - \mY^\star\|_F^2}{\sigma_{\min}^2(\mY^\star)}$ for   $i=2,\dots, N-2$. We now bound $\|\mW_1 - \mR_1^\top\mW_1^\star\|_F^2$ by
\vspace{-0.1cm}
\begin{eqnarray}
    \label{LOWER BOUND OF TWO DISTANCES IN DNN_7}
    &\!\!\!\!\!\!\!\!&\|\mW_1 - \mR_1^\top\mW_1^\star\|_F^2\nonumber\\
    &\!\!\!\!=\!\!\!\!& \|\mW_N^\star\cdots \mW_2^\star\mR_1 \mW_1 - \mW_N^\star\cdots \mW_1^\star\|_F^2\nonumber\\
    &\!\!\!\!\leq\!\!\!\!&2\|\mW_1\|^2\|\mW_N\cdots\mW_2 - \mW_N^\star\cdots\mW_2^\star\mR_1\|_F^2\nonumber\\
    &\!\!\!\!\!\!\!\!& + 2\|\mW_N\cdots \mW_1 -  \mW_N^\star\cdots \mW_1^\star\|_F^2\nonumber\\
    &\!\!\!\!\leq\!\!\!\!&\frac{20\|\mY^\star\|^2 \|\mY - \mY^\star\|_F^2}{\sigma_{\min}^2(\mY^\star)},
\end{eqnarray}
where the second inequality uses the fact that $\mY = \mW_N\cdots \mW_1 \mX$ and $\|\mA\mX\|_F = \|\mA\|_F$ for any $\mA$ since $\mX$ is whitened. Based on the preceding discussion and the definition of $\text{dist}^2({\rm W},{\rm W}^\star)$, we can conclude \eqref{LOWER BOUND OF TWO DISTANCES IN DNN_1}.

Finally, we can prove the other direction by
\vspace{-0.1cm}
\begin{eqnarray}
    \label{UPPER BOUND OF TWO DISTANCES IN DNN_2}
    &\!\!\!\!\!\!\!\!\!\!&\|\mY-\mY^\star\|_F^2\nonumber\\
     &\!\!\!\!\!=\!\!\!\!\!& \|\sum_{i=1}^N \mW_N^\star\cdots \mW_{i+1}^\star\mR_i(\mW_i \!-\! \mR_i^\top\mW_i^\star\mR_{i-1} )\mW_{i-1}\cdots\mX\|_F^2\nonumber\\
    &\!\!\!\!\!\leq\!\!\!\!\!& N\bigg(\sum_{i=2}^{N} \frac{9\|\mY^\star\|^2}{4}\|\mW_i-\mR_{i}^\top{\mW_i^{\star}}\mR_{i-1}\|_F^2\nonumber\\
    &\!\!\!\!\!\!\!\!\!\!&+ \|\mW_1 - \mR_1^\top\mW_1^\star\|_F^2\bigg)\leq\frac{9N}{4}\text{dist}^2({\rm W},{\rm W}^\star).
\end{eqnarray}

\end{proof}

\vspace{-0.45cm}
\paragraph*{Main results}
To establish the convergence rate of RGD, we require the loss function $L$ to satisfy a certain property. Given that our primary focus is the analysis of the local convergence, we will assume that the loss function behaves well only in a local region. Specifically, we will consider a category of loss functions that satisfies the so-called restricted correlated gradient (RCG) condition~\cite{Han20}:

\begin{definition}
We say the loss function $L(\cdot;\mY^\star)$ satisfies $\text{RCG}(\alpha,\beta,\calC)$ condition for $\alpha, \beta>0$ and the set $\mathcal{C}$ if
\begin{eqnarray}
    \label{RCG condition}
    &&\hspace{-1.5cm}\<\nabla L(\mY_1;\mY^\star)-\nabla L(\mY_2;\mY^\star), \mY_1 - \mY_2\>\nonumber\\
    &&\hspace{-1.2cm}\geq\alpha\|\mY_1 - \mY_2\|_F^2+\beta\|\nabla L(\mY_1;\mY^\star)-\nabla L(\mY_2;\mY^\star)\|_F^2
\end{eqnarray}
for any $\mY_1,\mY_2\in\calC$.
\end{definition}
The RCG condition is a generalization of the strong convexity. When $L$ represents the MSE loss, i.e., $L(\mY,\mY^\star) = \|\mY - \mY^\star\|_F^2$, which is commonly used in the convergence analysis of training deep linear networks \cite{zou2020global,arora2018convergence,zhu2020global}, it satisfies the RCG condition with $\alpha = \beta = 1$ and $\calC = \R^{d_y\times n}$. The RCG condition may also accommodate other loss functions such as the cross entropy (CE) loss.

Based on \Cref{RCG condition},  we can initially deduce the Riemannian regularity condition as an extension of the regularity condition found in matrix factorization \cite{Tu16,Zhu21TIT}, ensuring that gradients remain well-behaved within a defined region.
Specifically, we have
\begin{lemma} (Riemannian regularity condition)
\label{riemannian regularity condition}
Suppose the training data $(\mX,\mY^\star)$ is generated according to the data model in \Cref{def:data-model}. Also assume that the loss function $L$ in \eqref{DLN_LOSS_FUNCTION_1 general} adheres to the $\text{RCG}(\alpha,\beta,\calC)$ condition where $\calC  \triangleq  \{ \mY: \|\mY-\mY^\star\|_F^2\leq \frac{\alpha\beta\sigma_{\min}^2(\mY^\star)}{72(2N-1)^2(N^2-1)\kappa^2(\mY^\star)} \}$.  Under this assumption, for any ${\rm W}\in \{{\rm W}:  \text{dist}^2({\rm W},{\rm W}^\star)\leq  \frac{\alpha\beta\sigma_{\min}^2(\mY^\star)}{9(2N-1)(N^2-1)}  \}$, the function $g$ in \eqref{DLN_LOSS_FUNCTION_1 general} satisfies the Riemannian regularity condition as following:
\vspace{-0.25cm}
\begin{eqnarray}
    \label{Riemannian regularity condition eqn}
    &\!\!\!\!\!\!\!\!&\sum_{i=2}^{N} \bigg\<{\mW}_i - \mR_i^\top\mW_i^\star \mR_{i-1}, \calP_{\text{T}_{\mW_i} \text{St}}(\nabla_{\mW_i}g )\bigg\>\nonumber\\
    &\!\!\!\!\!\!\!\!&+ \<\mW_1 - \mR_1\mW_1^\star,  \nabla_{\mW_1}g \>\nonumber\\
    &\!\!\!\!\geq \!\!\!\!&\frac{\alpha}{16(2N-1)\kappa^2(\mY^\star)}\text{dist}^2({\rm W},{\rm W}^\star) + \frac{\beta}{(9N-5)\|\mY^\star\|^2 }\nonumber\\
    &\!\!\!\!\!\!\!\!&\times \bigg(\sum_{i=2}^{N}\|\calP_{\text{T}_{\mW_i} \text{St}}(\nabla_{\mW_i}g ) \|_F^2 + \|\mY^\star\|^2\|\nabla_{\mW_1}g\|_F^2\bigg).
\end{eqnarray}
\end{lemma}
\begin{proof} To begin with, we can derive $\|\mW_1\|^2\leq2\|\mW_1^\star\|^2+2\|\mW_1-\mR_1^\top\mW_1^\star\|^2\leq 2\|\mY^\star\|^2+2\text{dist}^2({\rm W},{\rm W}^\star) \leq 2\|\mY^\star\|^2+\frac{2\alpha\beta\sigma_{\min}^2(\mY^\star)}{9(2N-1)(N^2-1)}
\leq\frac{9\|\mY^\star\|^2}{4}$ where $\alpha\beta\leq \frac{1}{4}$ \cite{Han20} is used .

Next, through the gradients $\nabla_{\mW_i}g  = {\mW}_{i+1}^\top \cdots \mW_N^\top$ $ \nabla L(\mY;\mY^\star)\mX^\top\cdots {\mW}_{i-1}^\top, i\in [N]$, we need to derive
\vspace{-0.15cm}
\begin{eqnarray}
    \label{PROJECTED GRADIENT DESCENT SQUARED TERM of general function DNN 1}
    &\!\!\!\!\!\!\!\!&\|\nabla_{\mW_1}g\|_F^2\leq\|\nabla L(\mY;\mY^\star) - \nabla L(\mY^\star;\mY^\star)\|_F^2,\\
    \label{PROJECTED GRADIENT DESCENT SQUARED TERM of general function DNN 2}
    &\!\!\!\!\!\!\!\!&\|\nabla_{\mW_i}g\|_F^2\leq\|\mW_1\|^2\|\nabla L(\mY;\mY^\star)\|_F^2\nonumber\\
    &\!\!\!\!\leq\!\!\!\!&\frac{9\|\mY^\star\|^2}{4}\|\nabla L(\mY;\mY^\star) - \nabla L(\mY^\star;\mY^\star)\|_F^2,
\end{eqnarray}
where we employ $\nabla L(\mY^\star;\mY^\star) = 0$. Combing \eqref{PROJECTED GRADIENT DESCENT SQUARED TERM of general function DNN 1} and \eqref{PROJECTED GRADIENT DESCENT SQUARED TERM of general function DNN 2}, we can obtain
\vspace{-0.1cm}
\begin{eqnarray}
    \label{RIEMANNIAN FACTORIZATION SQUARED TERM UPPER BOUND of general function}
    &\!\!\!\!\!\!\!\!&\sum_{i=2}^{N}\|\calP_{\text{T}_{\mW_i} \text{St}}(\nabla_{\mW_i}g ) \|_F^2 + \|\mY^\star\|^2\|\nabla_{\mW_1}g\|_F^2\nonumber\\
    &\!\!\!\!\leq\!\!\!\!&\sum_{i=2}^{N}\|\nabla_{\mW_i}g  \|_F^2 + \|\mY^\star\|^2\|\nabla_{\mW_1}g\|_F^2\nonumber\\
    &\!\!\!\!\leq\!\!\!\!&\frac{9N-5}{4}\|\mY^\star\|^2\|\nabla L(\mY;\mY^\star) - \nabla L(\mY^\star;\mY^\star)\|_F^2,
\end{eqnarray}
in which the first inequality follows from the fact that for any matrix $\mB = \calP_{\text{T}_{L({\mX}_i)} \text{St}}(\mB) + \calP_{\text{T}_{L({\mX}_{i})} \text{St}}^{\perp}(\mB)$ where $\calP_{\text{T}_{L({\mX}_i)} \text{St}}(\mB)$ and $\calP_{\text{T}_{L({\mX}_{i})} \text{St}}^{\perp}(\mB)$ are orthogonal, we have $\|\calP_{\text{T}_{L({\mX}_i)} \text{St}}(\mB)\|_F^2\leq \|\mB\|_F^2$.

Before analyzing the lower bound of cross term in \eqref{RCG condition}, we need to establish the upper bound for the inner product between the orthogonal complement and the gradient as following:
\begin{eqnarray}
    \label{PROJECTION ORTHOGONAL IN RIEMANNIAN DNN UPPER BOUND of general function}
    T&\!\!\!\!=\!\!\!\!&\sum_{i=2}^{N}\<\calP^{\perp}_{\text{T}_{\mW_i} \text{St}}(\mW_i - \mR_i^\top\mW_i^\star \mR_{i-1}), \nabla_{\mW_i}g \>\nonumber\\
    &\!\!\!\!\leq\!\!\!\!&\sum_{i=2}^{N}\frac{1}{2}\|\mW_i\| \|\mW_i - \mR_i^\top\mW_i^\star \mR_{i-1}\|_F^2\nonumber\\
    &\!\!\!\!\!\!\!\!& \times\|\nabla L(\mY;\mY^\star) - \nabla L(\mY^\star;\mY^\star)\|_F\|\mW_1\| \nonumber\\
    &\!\!\!\!\leq\!\!\!\!&\frac{\beta}{4}\|\nabla L(\mY;\mY^\star) - \nabla L(\mY^\star;\mY^\star)\|_F^2\nonumber\\
    &\!\!\!\!\!\!\!\!&+\frac{9(N-1)\|\mY^\star\|^2}{16\beta}\sum_{i=2}^{N}\|\mW_i - \mR_i^\top\mW_i^\star \mR_{i-1}\|_F^4\nonumber\\
    &\!\!\!\!\leq\!\!\!\!&\frac{\beta}{4}\|\nabla L(\mY;\mY^\star) - \nabla L(\mY^\star;\mY^\star)\|_F^2\nonumber\\
    &\!\!\!\!\!\!\!\!&+\frac{9(N-1)}{16\beta\|\mY^\star\|^2}\text{dist}^4({\rm W},{\rm W}^\star),
\end{eqnarray}
where $\calP^{\perp}_{\text{T}_{\mW_i} \text{St}}(\mW_i- \mR_i^\top\mW_i^\star \mR_{i-1})
        =\frac{1}{2}\mW_i(\mW_i - \mR_i^\top\mW_i^\star $ $ \mR_{i-1})^\top  (\mW_i - \mR_i^\top\mW_i^\star \mR_{i-1})$ \cite{LiSIAM21} and  $\nabla L(\mY^\star;\mY^\star) = 0$.

Let us now introduce the notation $\mH \!\!= \!\!\mY^\star -\mW_N $ $ \cdots \mW_2  \mR_1^\top\mW_1^\star \mX + \sum_{i=1}^N\mW_N \cdots  \mW_{i+1}(\mW_i - \mR_i^\top \mW_i^\star $ $ \mR_{i-1})\mW_{i-1}\cdots \mW_1\mX$, enabling us to simplify the expression of the cross term within \eqref{Riemannian regularity condition eqn}. Then \eqref{Riemannian regularity condition eqn} can be rewritten as
\vspace{-0.1cm}
\begin{eqnarray}
    \label{RIEMANNIAN FACTORIZATION DNN CROSS TERM LOWER BOUND of general function}
    &\!\!\!\!\!\!\!\!&\sum_{i=2}^{N} \bigg\<\mW_i - \mR_i^\top\mW_i^\star \mR_{i-1}, \calP_{\text{T}_{\mW_i} \text{St}}(\nabla_{\mW_i}g )\bigg\> \nonumber\\
    &\!\!\!\!\!\!\!\!&+ \<\mW_1 - \mR_1^\top\mW_1^\star,  \nabla_{\mW_1}g \>\nonumber\\
    &\!\!\!\!=\!\!\!\!&\<\nabla L(\mY;\mY^\star) - \nabla L(\mY^\star;\mY^\star), \mY - \mY^\star + \mH  \> - T\nonumber\\
    &\!\!\!\!\geq\!\!\!\!&\alpha\|\mY-\mY^\star\|_F^2 + \frac{\beta}{4}\|\nabla L(\mY;\mY^\star) - \nabla L(\mY^\star;\mY^\star)\|_F^2\nonumber\\
    &\!\!\!\!\!\!\!\!&-\frac{9(N^2-1)}{16\|\mY^\star\|^2}\text{dist}^4({\rm W},{\rm W}^\star)\nonumber\\
    &\!\!\!\!\geq\!\!\!\!&\frac{\alpha}{16(2N-1)\kappa^2(\mY^\star)}\text{dist}^2({\rm W},{\rm W}^\star)+ \frac{\beta}{(9N-5)\|\mY^\star\|^2 } \nonumber\\
    &\!\!\!\!\!\!\!\!& \times \bigg(\sum_{i=2}^{N}\|\calP_{\text{T}_{\mW_i} \text{St}}(\nabla_{\mW_i}g ) \|_F^2 + \|\mY^\star\|^2\|\nabla_{\mW_1}g\|_F^2\bigg),
\end{eqnarray}
where the first inequality follows the RCG condition, \eqref{PROJECTION ORTHOGONAL IN RIEMANNIAN DNN UPPER BOUND of general function} and $\|\mH\|_F^2\leq \frac{9N(N-1)}{8\|\mY^\star\|^2}\text{dist}^4({\rm W},{\rm W}^\star)$  which is established through a mathematical transformation and the use of norm inequalities. The detailed proof for the upper bound of $\|\mH\|_F^2$ has been omitted here due to space limitations. In the last line, we leverage \eqref{RIEMANNIAN FACTORIZATION SQUARED TERM UPPER BOUND of general function}, \Cref{LOWER BOUND OF TWO DISTANCES IN DNN} and $\text{dist}^2({\rm W},{\rm W}^\star)\leq \frac{\alpha\beta\sigma_{\min}^2(\mY^\star)}{9(2N-1)(N^2-1)}$.

\end{proof}
\vspace{-0.1cm}
We note that according to \eqref{LOWER BOUND OF TWO DISTANCES IN DNN_1} in \Cref{LOWER BOUND OF TWO DISTANCES IN DNN}, the set $\{{\rm W}:  \text{dist}^2({\rm W},{\rm W}^\star)\leq  \frac{\alpha\beta\sigma_{\min}^2(\mY^\star)}{9(2N-1)(N^2-1)}  \}$ implies the region $\calC$.
By leveraging Riemannian regularity condition in \Cref{riemannian regularity condition} and the nonexpansiveness property of the polar decomposition-based retraction in \cite[Lemma 1]{LiSIAM21}, we ultimately reach the following conclusion:
\begin{theorem}
\label{Local Convergence of DNN General Stiefel_Theorem}
In accordance with the identical conditions outlined in \Cref{riemannian regularity condition}, we assume the initialization satisfies $\text{dist}^2({\rm W}^{(0)},{\rm W^\star})\leq \frac{\alpha\beta\sigma_{\min}^2(\mY^\star)}{9(2N-1)(N^2-1)}$. When employing the learning rate $\mu\leq\frac{2\beta}{(9N-5)\|\mY^\star\|^2}$ and $\gamma = \|\mY^\star\|^2$ in RGD \eqref{SGD_GRADIENT_DESCENT_1_2}, we have
\vspace{-0.2cm}
\begin{eqnarray}
    \label{Local Convergence of DNN General Stiefel_Theorem_1}
    \text{dist}^2({\rm W}^{(t+1)}, \! {\rm W}^\star) \! \leq \! \bigg(1 \! -\! \frac{\alpha\sigma_{\min}^2(\mY^\star)\mu}{8(2N-1)}\bigg)\text{dist}^2({\rm W}^{(t)},\! {\rm W}^\star).\nonumber
\end{eqnarray}
\end{theorem}

Our results reveal that the RGD demonstrates a linear convergence rate with polynomial decay concerning $N$. In addition, by \Cref{LOWER BOUND OF TWO DISTANCES IN DNN}, we can easily obtain $\|\mY^{(t)} - \mY^\star\|_F^2\leq (1 \! -\! \frac{\alpha\sigma_{\min}^2(\mY^\star)\mu}{8(2N-1)})^t\frac{\alpha\beta\sigma_{\min}^2(\mY^\star)}{8N^2 - 12N +4}$. It is worth noting that, through analogous analysis, \Cref{Local Convergence of DNN General Stiefel_Theorem} can be extended to a broader scenario wherein $\mW_j^\star$ takes on an arbitrary matrix, and matrices $\mW_i^\star$ exhibit row orthogonality for $i < j$ and column orthogonality for $i > j$.
Moreover, we emphasize that our focus is primarily on the local convergence property of the RGD, and does not cover initialization methods extensively investigated in prior research, such as those discussed in \cite{saxe2013exact, he2015delving, xiao2018dynamical, bartlett2018gradient}.

The research most closely related to our work is the convergence analysis of gradient descent in deep linear neural models across multiple layers, as demonstrated in \cite{arora2018convergence}, wherein the MSE loss function is taken into account. It is established that $\|\mY^{(t)} - \mY^\star\|_F^2\leq \epsilon$ can be deduced when $t\geq \Omega(\frac{N^3\|\mY^\star\|_F^{6}}{c^2}\log(\frac{1}{\epsilon}) )$, provided that $c \!\leq \! \sigma_{\min}(\mY^\star)$. Applying a similar derivation as presented in \cite[Theorem 1]{arora2018convergence}, utilizing both \Cref{Local Convergence of DNN General Stiefel_Theorem} and \Cref{LOWER BOUND OF TWO DISTANCES IN DNN}, we can infer that $t\geq \Omega(N^2\kappa^2(\mY^\star) \log(\frac{1}{\epsilon}))$ in the RGD is sufficient to meet the same requirement. This further highlights the convergence advantage of the RGD.

\vspace{-0.3cm}
\section{Experiments}
\label{Experiment results}

In this section, we conduct experiments to compare the performance of the RGD with gradient descent (GD). Specifically, we concentrate on the multi-class classification task using the MNIST dataset, where the input feature dimension is $784$, while the output feature is $10$, represented as ${\vy_i}\in\R^{10}$. In this representation, each ${\vy_i}$ is designed such that it holds a value of $1$ solely at the position that aligns with its categorical label, leaving the other positions assigned to $0$. For our model architecture, we can deploy an multi-layer perceptron (MLP) with $N$ layers where $\mW_1\in\R^{{100}\times {784}}$, $\{\mW_i\}_{i=2}^{N-2}\in\R^{{100}\times {100}}$, $\mW_{N-1}\in\R^{{50}\times {100}}$, and $\mW_N\in\R^{{10}\times {50}}$. Each layer in the MLP is connected by a linear or rectified linear unit (Relu) activation function. While our theoretical result is only established for linear networks, we will also test the performance on a nonlinear MLP without bias terms.

We apply the GD and RGD for the CE loss function \cite{zhou2022all} in combination with the softmax function to train MLP models. The weight matrices are initialized using orthogonal initialization \cite{saxe2013exact}, except for $\mW_1$ in the RGD, which follows a uniform distribution within the range of $(-\frac{1}{\sqrt{784}}, \frac{1}{\sqrt{784}})$ \cite{he2015delving}. In addition, we perform a grid search to fine-tune the hyperparameters ($\mu$ and $\gamma$).
\vspace{-0.2cm}
\begin{figure}[htbp]
\centering
\includegraphics[width=6cm, keepaspectratio]%
{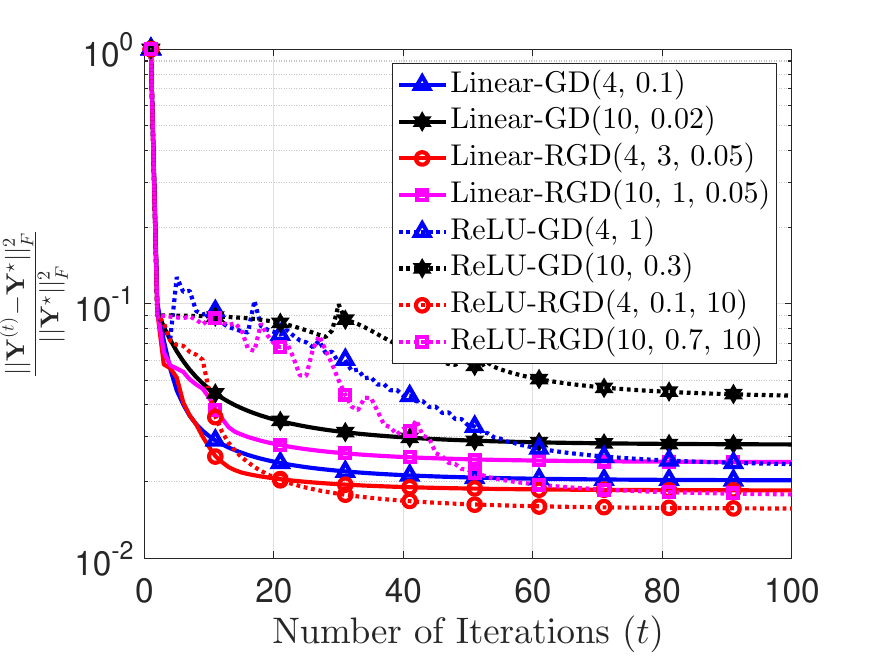}
\caption{Convergence analysis for GD($N$, $\mu$) and RGD($N$, $\mu$, $\gamma$) with different activation functions and $N$.}
\label{Convergence analysis for GD and RGD}
\end{figure}
\vspace{-0.3cm}
In Figure~\ref{Convergence analysis for GD and RGD}, it is evident that RGD achieves a quicker convergence compared to GD. Notably, with an increase in the number of layers, the convergence rate decreases in alignment with our theoretical analysis. Moreover, due to the impact of nonlinear activation functions, algorithms that employ ReLU demonstrate a comparatively slower convergence rate compared to those utilizing linear activation functions. However, despite this, the error of the RGD with Relu outperforms that of the RGD using a linear activation function.

\vspace{-0.25cm}
\section{Conclusion}
\label{conclusion}

In this letter, we have provided a convergence analysis of the Riemannian gradient descent for a specific class of loss functions within orthonormal deep linear neural networks. Remarkably, our analysis guarantees a linear convergence rate, provided appropriate initialization.  This will serve as a stepping stone for future explorations of training nonlinear orthonormal deep neural networks with adaptive learning rates.

%

\newpage
\balance

\end{document}